\pdfoutput=1
\documentclass{article}
\usepackage[preprint]{neurips_2026}
\usepackage[T1]{fontenc}

\usepackage[utf8]{inputenc}
\usepackage{amsmath,amssymb,amsthm,graphicx,booktabs,xcolor,hyperref,url}
\hypersetup{colorlinks,linkcolor=blue!50!black,citecolor=blue!50!black,urlcolor=blue!50!black,
  pdftitle={Two-Sample Testing for Inhomogeneous Random Graphs in Non-Integral Lr Norms},pdfauthor={Soham Dan}}
\usepackage[section]{placeins}
\newtheorem{theorem}{Theorem}
\newtheorem{lemma}{Lemma}
\newcommand{\E}{\mathbb E}
\newcommand{\Prob}{\mathbb P}
\newcommand{\ind}{\mathbf 1}
\newcommand{\norm}[2]{\left\|#1\right\|_{#2}}

\title{Two-Sample Testing for Inhomogeneous\\ Random Graphs in Non-Integral $L_r$ Norms}
\author{Soham Dan}

\begin{document}
\maketitle

\begin{abstract}
Testing whether two populations of networks share the same edge probabilities is a basic problem in network inference. How hard it is depends on the norm used to measure the difference. For the inhomogeneous Erd\H{o}s--R\'enyi (IER) model, the optimal sample complexity is known for every integer $L_r$ norm and for $1\le r<2$. For non-integral $r>2$, however, the known upper and lower bounds do not match, and the lower bound was conjectured to be tight. We study this gap for two-sample testing on aligned vertices. We propose a test that runs two published statistics, of orders $2$ and $\lceil r\rceil$, on the same data and rejects if either one rejects. Its thresholds come from H\"older interpolation, so that both statistics have the same sample cost. We prove that this test attains the conjectured rate. Combined with earlier results, this shows that for every fixed $r\ge1$ the minimax sample complexity is of order $n^{\max\{4/r-1,\,2/r\}}/\epsilon^2$, even when the separation changes with $n$. In simulations with $n$ between 32 and 256, the number of graphs needed for 80\% power at level $0.05$ grows with $n$ at a rate consistent with the theory. For $r=2.5$, for example, the fitted exponent is $0.78$, against the theoretical value $0.8$. Interestingly, the two statistics split the work as the interpolation argument suggests: the higher-order statistic is more powerful when only a few edges change, and the $L_2$ statistic when many edges change.
\end{abstract}

\section{Introduction}
Many scientific questions ask whether two groups of networks differ, for instance brain networks recorded from two groups of subjects. When the vertices are aligned across networks and the edges are independent, every possible edge supplies repeated Bernoulli observations in each group, and the question becomes whether the two groups have the same edge probabilities. How hard this is depends on how a difference is measured. An entrywise $L_r$ norm, the $r$-th root of the sum of the $r$-th powers of the absolute differences in edge probabilities, weighs different kinds of change differently. With small $r$, many small changes spread over the network count heavily. With large $r$, a few large changes on a small set of edges dominate. A test that is reliable against every alternative at a given $L_r$ distance must therefore detect both diffuse and concentrated changes.

\citet{chatterjee2023} studied this problem for the inhomogeneous Erd\H{o}s--R\'enyi (IER) model and found the optimal sample complexity, that is, the number of graphs per group that the best test needs, in two cases. For integer $r\ge2$ it is of order $n^{2/r}/\epsilon^2$, and for every real $1\le r<2$ it is of order $n^{4/r-1}/\epsilon^2$, where $n$ is the number of vertices and $\epsilon$ is the separation between the two groups. For non-integral $r>2$, their Corollary~2.2 gives a lower bound of order $n^{2/r}/\epsilon^2$ and an upper bound of order $n^{4/r-2/\lceil r\rceil}/\epsilon^2$, and their Remark~2.2 conjectures that the lower bound gives the correct rate. The gap comes from converting every alternative into the next integer norm $\lceil r\rceil$, a step that loses a power of $n$. This leaves a natural research question: can a test detect every alternative at non-integral $L_r$ distance $\epsilon$ using only order $n^{2/r}/\epsilon^2$ graphs per group?

In this note, we answer this question positively. Rather than embedding the whole alternative class into one integer norm, we use interpolation between norms to split it between two tests. One test uses the $L_2$ statistic of \citet{chatterjee2023}, and the other uses their $L_d$ statistic with $d=\lceil r\rceil$. We choose the two thresholds so that both tests have the same sample cost. Rejecting when either test rejects then attains the conjectured rate, even when both statistics are computed on the same observations. For the lower bound, we use the dense and sparse random constructions of \citet{chatterjee2023}. Our proof is self-contained. Its bounds hold uniformly over the null and alternative classes, including alternatives that lie exactly on the boundary, and they allow integer rounding of the sample size and separations that change with $n$. Table~\ref{tab:rates} summarizes what was known and what is new.

\begin{table}[ht]
\centering
\caption{Sample complexity of two-sample testing in the aligned IER model, up to constants that depend on $r$ and the target error. Every entry is divided by $\epsilon^2$. The last row is the gap closed in this note. Here $\Theta(\cdot)$ means ``of the same order as'', while $\Omega(\cdot)$ and $O(\cdot)$ give lower and upper bounds of that order.}\label{tab:rates}
\begin{tabular}{lll}
\toprule
Norm & Known before & This note \\
\midrule
$1\leq r<2$ & $\Theta(n^{4/r-1})$ & unchanged \\
Integer $r\geq2$ & $\Theta(n^{2/r})$ & unchanged \\
Non-integral $r>2$ & $\Omega(n^{2/r})$ and $O(n^{4/r-2/\lceil r\rceil})$ & $\Theta(n^{2/r})$ \\
\bottomrule
\end{tabular}
\end{table}

We complement the theory with simulations. Using an exact simulator, we check that the number of graphs needed for 80\% power at level $0.05$ grows with $n$ at a rate consistent with Theorem~\ref{thm:main}, for $r\in\{1.5,2,2.5,3.5\}$. We also show how the two statistics share the work across alternatives that perturb different numbers of edges.

Our contribution is deliberately narrow. We reuse the statistics and lower-bound constructions of \citet{chatterjee2023}, and the only new ingredient is the interpolation-based choice of two thresholds. Beyond symmetry and independent aligned edges, no property of graphs enters the argument. The result is really a statement about comparing the mean vectors of two sets of independent Bernoulli observations indexed by unordered edges. Section~\ref{sec:model} states the setting and the main result, Section~\ref{sec:construction} explains the test and why it works, Section~\ref{sec:experiments} reports the simulations, and Section~\ref{sec:limitations} discusses related work and limitations.

\section{Setting and main result}\label{sec:model}
Let $\mathcal P_n$ be the set of symmetric $n\times n$ matrices with zero diagonal and off-diagonal entries in $[0,1]$. A matrix $P\in\mathcal P_n$ defines the model $\mathrm{IER}(P)$: a random graph on the vertices $1,\ldots,n$ in which each edge $\{i,j\}$ appears independently with probability $p_{ij}$. We observe $m$ independent graphs $G_1,\ldots,G_m\sim\mathrm{IER}(P)$ and $m$ independent graphs $H_1,\ldots,H_m\sim\mathrm{IER}(Q)$ on the same labelled vertices. We write
\[
 E_n=\{(i,j):1\leq i<j\leq n\},\quad M=|E_n|=n(n-1)/2,\quad N=2M
\]
for the set of unordered edges, its size, and the number of off-diagonal matrix entries. For finite $r\geq1$, we measure the difference between the two groups by the entrywise norm
\begin{equation}\label{eq:norm}
 \norm{P-Q}{r}^r=2\sum_{e\in E_n}|p_e-q_e|^r,
\end{equation}
where the factor two counts both symmetric entries, as in \citet{chatterjee2023}. We write $h_e=p_e-q_e$ and test
\[
 H_0:\ P=Q\qquad\text{against}\qquad H_1:\ \norm{P-Q}{r}\geq\epsilon.
\]
A test $\phi$ is a function of the data with values $1$ (reject $H_0$) and $0$ (accept $H_0$). We judge a test by its worst-case risk, the largest type~I error over the null plus the largest type~II error over the alternative:
\begin{equation}\label{eq:risk}
 R_{n,m,r,\epsilon}(\phi)=\sup_{P\in\mathcal P_n}\Prob_{P,P}(\phi=1)
 +\sup_{\substack{P,Q\in\mathcal P_n\\\norm{P-Q}{r}\geq\epsilon}}\Prob_{P,Q}(\phi=0),
 \qquad R^*_{n,m,r,\epsilon}=\inf_\phi R_{n,m,r,\epsilon}(\phi).
\end{equation}
The two suprema may be attained by different matrices. The minimax risk $R^*$ is the smallest worst-case risk that any test can achieve. For a fixed target $\beta\in(0,1)$, the sample complexity $m^*_\beta(n,r,\epsilon)$ is the smallest positive integer $m$ for which $R^*_{n,m,r,\epsilon}\leq\beta$. Put differently, it is the number of graphs per group that the best test needs to keep the sum of its two errors below $\beta$. Table~\ref{tab:notation} collects the main notation.

\begin{table}[ht]
\centering
\caption{Main notation.}\label{tab:notation}
\begin{tabular}{ll}
\toprule
Symbol & Meaning \\
\midrule
$n$, $m$ & number of vertices; number of graphs in each group \\
$M$, $N$ & number of unordered edges $n(n-1)/2$; number of off-diagonal entries $2M$ \\
$h_e=p_e-q_e$ & difference of the edge probabilities of edge $e$ \\
$\epsilon$ & separation in the $L_r$ norm \eqref{eq:norm} \\
$d=\lceil r\rceil$ & the next integer above a non-integral $r$ \\
$\lambda(r)$, $A_r$ & rate exponent $\max\{4/r-1,2/r\}$; rate $n^{\lambda(r)}/\epsilon^2$ \\
$\tau=m/A_r$ & number of graphs per group, in units of the rate \\
$a\asymp b$ & $a/b$ stays bounded away from $0$ and $\infty$ \\
\bottomrule
\end{tabular}
\end{table}

\begin{theorem}\label{thm:main}
Fix a finite $r\geq1$ and let $\epsilon_n>0$ be any sequence such that
\[
 A_r=\frac{n^{\lambda(r)}}{\epsilon_n^2}\longrightarrow\infty,
 \qquad \lambda(r)=\max\{4/r-1,2/r\},\qquad n\longrightarrow\infty.
\]
Then the alternative class is nonempty for all large $n$, and
\begin{align}
 m/A_r\longrightarrow\infty&\quad\Longrightarrow\quad R^*_{n,m,r,\epsilon_n}\longrightarrow0,\label{eq:upper}\\
 m/A_r\longrightarrow0&\quad\Longrightarrow\quad R^*_{n,m,r,\epsilon_n}\longrightarrow1.\label{eq:lower}
\end{align}
Moreover, for every fixed $\beta\in(0,1)$ there are constants $0<c<C$, depending only on $r$ and $\beta$, such that $cA_r\leq m^*_\beta(n,r,\epsilon_n)\leq CA_r$ for all large $n$.
\end{theorem}

Informally, $A_r$ is, up to constant factors, the number of graphs per group that the problem requires. With many more than $A_r$ graphs, some test makes both errors small, and with many fewer, even the best test does little better than guessing. The exponent $\lambda(r)$ equals $4/r-1$ for $r\leq2$ and $2/r$ for $r\geq2$, so the problem becomes easier as $r$ grows. The proof is in Appendix~\ref{app:proofs}. Only the upper bound for non-integral $r>2$ is new, and the other rates and both lower-bound constructions are due to \citet{chatterjee2023}. We treat $r$ as fixed and do not track sharp constants, so growing $r$, the supremum norm, and regimes in which $A_r$ stays bounded are outside the scope of Theorem~\ref{thm:main}.

\section{The test and why it works}\label{sec:construction}
\subsection{The integer-order statistic}
\citet{chatterjee2023} detect differences in the $L_s$ norm, for an integer $s\geq2$, with the following statistic. Assume $m\geq s$, put $K_s=\lfloor m/s\rfloor$, and split the first $sK_s$ graph pairs into $s$ consecutive blocks of $K_s$ pairs each. Any remaining pairs are not used. For block $\ell$ and edge $e$, let
\[
 Y_{\ell,e}=\frac1{K_s}\sum_{t=(\ell-1)K_s+1}^{\ell K_s}(G_{t,e}-H_{t,e}),\qquad
 U_s=\sum_{e\in E_n}\begin{cases}
 \prod_{\ell=1}^sY_{\ell,e},&s\ \text{even},\\
 (\prod_{\ell=1}^{s-1}Y_{\ell,e})|Y_{s,e}|,&s\ \text{odd}.
 \end{cases}
\]
We write $W_e$ for the term of edge $e$. Up to the factor $K_s^s$, $U_s$ is the statistic of \citet{chatterjee2023}.

The construction of $U_s$ rests on one observation. Each block mean $Y_{\ell,e}$ is an unbiased estimate of $h_e$, and block means from different blocks are independent, so the product of $s$ of them has expectation $h_e^s$. Summing over edges therefore estimates $\sum_e h_e^s$, which for even $s$ is half of $\norm{P-Q}{s}^s$. For odd $s$, the term $h_e^s$ is negative when $h_e<0$, so positive and negative differences could cancel in the sum. The last factor is therefore replaced by its absolute value. Only the last factor is changed, because replacing every factor by its magnitude would give a nonzero mean under the null. In both cases, block independence gives $\E W_e\geq|h_e|^s$, with equality for even $s$, and $\E W_e=0$ when $h_e=0$. The second moment is
\[
 \E W_e^2=(h_e^2+v_e)^s,\qquad
 v_e=\{p_e(1-p_e)+q_e(1-q_e)\}/K_s\leq(2K_s)^{-1},
\]
where $v_e$ is the variance of one block mean.

Given a separation $\eta>0$, the test rejects when $U_s$ exceeds a quarter of $\eta^s$:
\begin{equation}\label{eq:integer-test}
 \psi_{s,\eta}=\ind\{U_s\geq\eta^s/4\}.
\end{equation}
Under the null, $U_s$ has mean zero, and under an alternative with $\norm{P-Q}{s}\geq\eta$ its mean is at least $\eta^s/2$, so the threshold sits halfway between the two. Chebyshev's inequality, applied under the null and under the alternative with the moments above, gives
\begin{equation}\label{eq:integer-main-bound}
 R_{n,m,s,\eta}(\psi_{s,\eta})\leq C_s\sum_{j=1}^s
 \left(\frac{M^{1/s}}{K_s\eta^2}\right)^j,
\end{equation}
where the constant $C_s$ depends only on $s$. We call $M^{1/s}/\eta^2$ the sample cost of the test: the risk is small once $K_s$, and hence $m$, is large compared with it. Because the argument uses only means and variances, the bound holds uniformly over all matrices, including probabilities equal to zero or one and alternatives exactly on the boundary.

\subsection{Splitting the alternatives between two tests}
The key step is an interpolation inequality between norms.
\begin{lemma}\label{lem:interpolation}
Let $2<r<d$, and set
\[
 \alpha=\frac{1/r-1/d}{1/2-1/d},\qquad
 \eta_2=\epsilon N^{1/4-1/(2r)},\qquad
 \eta_d=\epsilon N^{1/(2d)-1/(2r)}.
\]
Every vector $x\in\mathbb R^N$ with $\norm{x}{r}\geq\epsilon$ satisfies
$\norm{x}{2}\geq\eta_2$ or $\norm{x}{d}\geq\eta_d$.
\end{lemma}
\begin{proof}
H\"older's inequality (Appendix~\ref{app:proofs}) gives $\norm{x}{r}\leq\norm{x}{2}^{\alpha}\norm{x}{d}^{1-\alpha}$.
Since $\alpha/2+(1-\alpha)/d=1/r$, the thresholds satisfy
$\eta_2^\alpha\eta_d^{1-\alpha}=\epsilon$. If both $\norm{x}{2}<\eta_2$ and $\norm{x}{d}<\eta_d$, the right-hand side would be strictly smaller than $\epsilon$, which contradicts $\norm{x}{r}\geq\epsilon$.
\end{proof}

Lemma~\ref{lem:interpolation} has a concrete meaning. Consider a difference spread evenly over $k$ unordered edges, so that $x$ has $2k$ nonzero entries of equal size and $\norm{x}{r}=\epsilon$. A short calculation gives
\[
 \frac{\norm{x}{2}}{\eta_2}=\Big(\frac{2k}{\sqrt N}\Big)^{1/2-1/r},\qquad
 \frac{\norm{x}{d}}{\eta_d}=\Big(\frac{2k}{\sqrt N}\Big)^{1/d-1/r}.
\]
Since $r>2$ and $d>r$, the first ratio is at least one exactly when $2k\geq\sqrt N$, and the second exactly when $2k\leq\sqrt N$. The $L_2$ test is therefore responsible for diffuse differences on more than about $n/2$ edges, and the $L_d$ test for concentrated differences on fewer edges (Figure~\ref{fig:interpolation}). The two ratios meet at $k=\sqrt N/2\approx n/2$, which is essentially the support size of the sparse lower-bound construction in Section~\ref{sec:lower}, suggesting that the hardest alternatives sit where the two tests hand over.

\begin{figure}[t]
\centering\includegraphics[width=.6\linewidth]{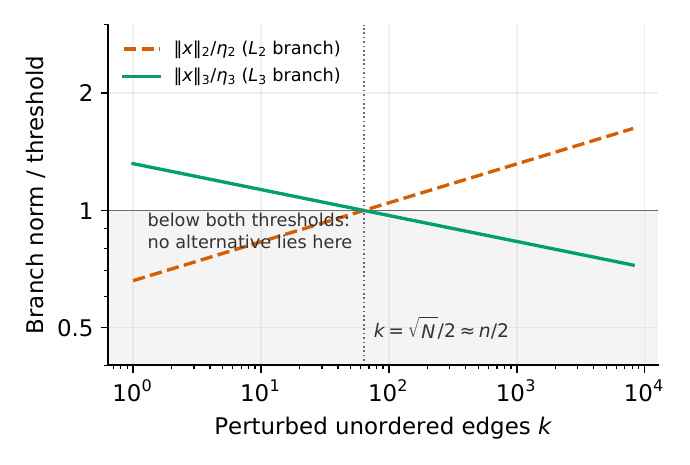}
\caption{Which test is responsible for which alternative, for $n=128$ and $r=2.5$ ($d=3$). Each alternative spreads a difference of $L_r$ norm $\epsilon$ evenly over $k$ edges. Above the horizontal line, a test faces an alternative at least as large as its threshold. At every $k$ at least one of the two tests does, and the hand-over happens at $k=\sqrt N/2$.}
\label{fig:interpolation}
\end{figure}

We apply the lemma to the $N$ off-diagonal entries of $P-Q$. For non-integral $r>2$, we take $d=\lceil r\rceil$ and define the union test
\begin{equation}\label{eq:union}
 \phi_{r,\epsilon}=\psi_{2,\eta_2}\vee\psi_{d,\eta_d},
\end{equation}
which rejects when either test rejects. With the thresholds of Lemma~\ref{lem:interpolation}, both tests have the same sample cost in~\eqref{eq:integer-main-bound}:
\begin{equation}\label{eq:balance}
 \frac{M^{1/2}}{\eta_2^2}=2^{-1/2}\frac{N^{1/r}}{\epsilon^2},
 \qquad \frac{M^{1/d}}{\eta_d^2}=2^{-1/d}\frac{N^{1/r}}{\epsilon^2}.
\end{equation}
Under the null, the type~I error of the union is at most the sum of the two type~I errors. Under any alternative, Lemma~\ref{lem:interpolation} guarantees that at least one of the two tests faces an alternative in its own norm, and the union accepts only if that test also accepts. Together with~\eqref{eq:integer-main-bound}, \eqref{eq:balance}, the bound $K_s\geq m/(2s)$, and $N\leq n^2$, these two facts prove~\eqref{eq:upper}, as detailed in Appendix~\ref{app:proofs}. Neither fact requires the two statistics to be independent, so they can be computed on the same observations. The same calculation works for any fixed integer $d>r$, and the ceiling is simply the smallest choice.

\begin{center}
\fbox{\begin{minipage}{0.94\linewidth}
\textbf{The union test for a non-integral $r>2$.}
Given $G_{1:m}$, $H_{1:m}$, and a separation $\epsilon>0$:
\begin{enumerate}
\item Set $d=\lceil r\rceil$ and compute the thresholds $\eta_2$ and $\eta_d$ of Lemma~\ref{lem:interpolation}. This requires $m\geq d$.
\item Compute $U_2$ and $U_d$ on the same observations, each with its own consecutive equal-sized blocks and discarding its own leftover pairs.
\item Reject $H_0$ if $U_2\geq\eta_2^2/4$ or $U_d\geq\eta_d^d/4$.
\end{enumerate}
\end{minipage}}
\end{center}

For $1\leq r<2$, we use $s=2$ with $\eta=\epsilon N^{1/2-1/r}$, and for integer $r\geq2$ we use~\eqref{eq:integer-test} directly with $s=r$ and $\eta=\epsilon$. When $r=2$ this is a single test. Rounding affects only constants, because $K_s\geq m/(2s)$ whenever $m\geq s$.

\subsection{Why no test can do better}\label{sec:lower}
The lower bounds come from two random constructions of \citet{chatterjee2023}. In each one, $P$ is drawn at random in a way that makes it hard to tell apart from $Q$ unless there are enough graphs. We write $L$ for the likelihood ratio of the data under the random alternative relative to the null, and $\E_0$ for expectation under the null. In the dense construction, $Q_e=1/2$ and $P_e=1/2+\gamma_e\delta$ with independent uniform random signs $\gamma_e$ and $\delta=\epsilon N^{-1/r}$. Its likelihood ratio has the exact second moment
\[
 \E_0L^2=\left\{\frac{(1+4\delta^2)^m+(1-4\delta^2)^m}{2}\right\}^{M},
 \qquad \log\E_0L^2\leq8Mm^2\delta^4.
\]
For $r\leq2$, $\E_0L^2$ approaches one when $m\epsilon^2/n^{4/r-1}\to0$. In the sparse construction, used for $r\geq2$, a uniformly random set of $k=\lfloor n/2\rfloor$ unordered edges receives the positive shift $\delta=\epsilon(2k)^{-1/r}$. If $Z$ is the number of edges shared by two independent random supports and $b=(1+4\delta^2)^m$, then
\[
 \E_0L^2=\E b^Z\leq\exp\{(k^2/M)(b-1)\}.
\]
Here $Z$ has a hypergeometric distribution. Expanding $\E b^Z$ in factorial moments of $Z$ gives the bound, which approaches one when $m\epsilon^2/n^{2/r}\to0$. When $\E_0L^2$ is close to one, the data have nearly the same distribution under the null and under the random alternative. Formally, the two error probabilities of any test sum to at least one minus the total-variation distance between these two distributions, which is the largest difference between the probabilities they assign to a single event. This distance is at most $\tfrac12\sqrt{\E_0L^2-1}$, so $R^*\geq1-\tfrac12\sqrt{\E_0L^2-1}$, which proves~\eqref{eq:lower}. Keeping the nonlinear form of both bounds also gives the constants for a fixed target error $\beta$. The condition $A_r\to\infty$ ensures that the perturbed probabilities stay in $[0,1]$ for large $n$.

\section{Simulations}\label{sec:experiments}
We use simulations to check the implementation of the union test on shared data and to see whether the number of graphs needed grows with $n$ at the rate of Theorem~\ref{thm:main}. We also use them to show how the two statistics share the work across different alternatives. Simulations cannot prove a minimax statement, so we treat them as a check on the implementation and an illustration of the theory.

\subsection{Setup}
\paragraph{Exact simulation.}
The sum of an edge's indicators over $c$ graphs is a binomial count, so we simulate block sums directly as independent binomial counts. The $s=2$ and $s=d$ statistics use different block boundaries, so we cut the observation indices at every boundary of either statistic, draw binomial counts on the resulting segments, and add them up for each statistic. This reproduces exactly the joint distribution of the two statistics computed on the same graphs, even when $m$ is not divisible by either block count. Drawing separate blocks for each statistic would instead simulate two independent experiments.

\paragraph{Alternatives and design.}
We use the two lower-bound constructions of Section~\ref{sec:lower}, both with background probability $1/2$: dense random-sign shifts for $r\leq2$, and a hidden support of $k=\lfloor n/2\rfloor$ positively shifted edges for $r\geq2$. The signs or supports are redrawn in every replicate, so the reported power is averaged over the construction. The null is $P=Q=1/2$. The main study has 324 cells, 286 for scaling and 38 for the support study of Section~\ref{sec:branches}. For scaling, we use $n\in\{32,64,128,256\}$, $r\in\{1.5,2,2.5,3.5\}$ and $\epsilon=0.75$, and we choose $m$ on a grid of $\tau=m\epsilon^2/n^{\lambda(r)}=m/A_r$ from $1/4$ to $32$, which gives between 7 and 587,166 graphs per group. Each cell has 500 replicates, except 32 null cells with 1,000.

\paragraph{Setting the threshold.}
We call the test with the thresholds~\eqref{eq:integer-test}--\eqref{eq:union} the theorem test. Its thresholds depend on $\epsilon$ and are designed to make the sum of the two errors small, but they do not fix the type~I error at a nominal level such as $5\%$. For comparisons at equal size, we also use an oracle calibration. It estimates the null distribution of each statistic from 999 independent simulations of the null used in the experiment and rejects at level $0.05$. For the union, the level is split equally between the two statistics (Bonferroni). This calibration uses the known null, so it is available only in simulations. Section~\ref{sec:limitations} describes a feasible calibration that is valid in finite samples. Finally, the practical normal calibration of \citet{chatterjee2023}, which standardizes each statistic by an estimate of its null variance, is reported in Appendix~\ref{app:results}. Calibration and evaluation always use independent random draws.

\subsection{How the number of graphs needed scales with \texorpdfstring{$n$}{n}}\label{sec:scaling}
Figure~\ref{fig:scaling} shows the rejection rate of the union test at level $0.05$ as a function of $\tau=m/A_r$. For every $r$, the power curves of the four graph sizes nearly coincide, even though at the same $\tau$ the number of graphs for $n=256$ is between about 3 and 32 times that for $n=32$. The null rejection stays between $0.02$ and $0.08$. This is what Theorem~\ref{thm:main} predicts: the number of graphs needed scales as $A_r$.

\begin{figure}[ht]
\centering\includegraphics[width=\linewidth]{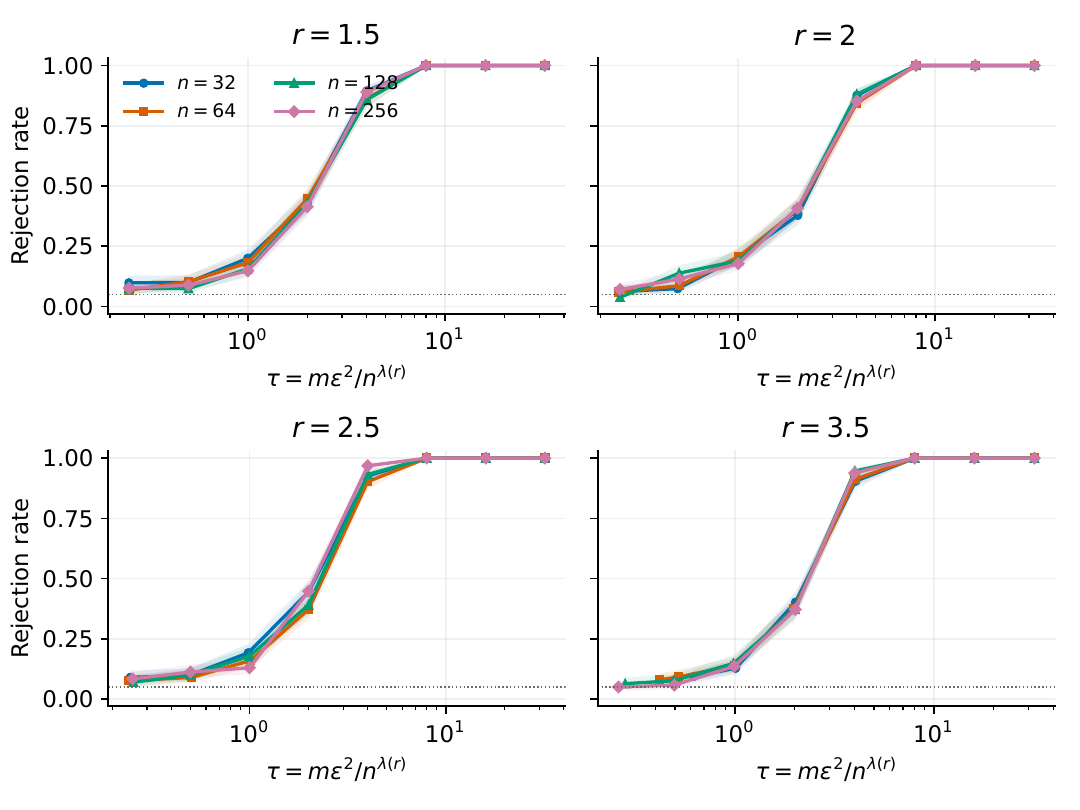}
\caption{Rejection rate of the union test at level $0.05$ (oracle calibration) against $\tau=m/A_r$, with $\epsilon=0.75$. Solid lines with filled markers show power under the dense construction ($r\leq2$) or the hidden support of $k=\lfloor n/2\rfloor$ edges ($r>2$). Dashed lines with open markers show the null $P=Q=1/2$, which stays near the dotted line at $0.05$. Colors and markers identify $n$. Bands are pointwise 95\% Wilson intervals over 500 replicates (1,000 for the null at $\tau\approx2$ and $8$), conditional on the calibration sample.}
\label{fig:scaling}
\end{figure}

To quantify the scaling, we find for each $n$ the number of graphs $m^*$ at which power first reaches $0.8$. We use a monotone (isotonic) fit to the power curve and log-linear interpolation, and then regress $\log m^*$ on $\log n$. Table~\ref{tab:slopes} shows that the fitted slopes are close to $\lambda(r)$ in all five series, and every 95\% bootstrap interval contains $\lambda(r)$. The same holds for the theorem test itself when we track one minus the sum of its null rejection rate and its miss rate, instead of its power alone. At the grid point nearest $\tau=8$, the theorem test has power between $0.996$ and $1.000$ on the non-integral cells, with null rejection between $0.000$ and $0.006$.

\begin{table}[ht]
\centering
\caption{Fitted slope of $\log m^*$ against $\log n$ over $n=32,\ldots,256$, where $m^*$ is the number of graphs per group at which the tracked quantity first reaches $0.8$ ($\epsilon=0.75$). Brackets give 95\% percentile intervals from 400 bootstrap samples that resample evaluation and calibration replicates within every cell. The last column tracks one minus the sum of the null rejection rate and the miss rate of the theorem test. This analysis was added after the prespecified one in Appendix~\ref{app:transitions}.}\label{tab:slopes}
\begin{tabular}{llcll}
\toprule
$r$ & Construction & $\lambda(r)$ & Power at level $0.05$ & Theorem test, $1-$error sum \\
\midrule
1.5 & dense  & 1.667 & 1.672 [1.641, 1.724] & 1.663 [1.621, 1.703] \\
2   & dense  & 1.000 & 1.003 [0.969, 1.041] & 1.017 [0.963, 1.076] \\
2   & sparse & 1.000 & 0.981 [0.948, 1.011] & 0.974 [0.911, 1.026] \\
2.5 & sparse & 0.800 & 0.776 [0.752, 0.804] & 0.792 [0.763, 0.823] \\
3.5 & sparse & 0.571 & 0.552 [0.528, 0.574] & 0.572 [0.539, 0.604] \\
\bottomrule
\end{tabular}
\end{table}

We planned a different analysis in advance, and we report it in Appendix~\ref{app:transitions}. It tracked the power of the theorem test alone, whose null rejection falls from between $0.34$ and $0.75$ at $\tau\leq1$ to nearly zero at $\tau\geq8$. Its power therefore reaches $0.8$ partly through false rejections, which produces a misleading slope of $1.07$ for $r=2.5$ and a censored crossing for $r=3.5$. Tracking power at a fixed level, or the error sum, removes this effect.

\subsection{How the two statistics share the work}\label{sec:branches}
To see the mechanism of Lemma~\ref{lem:interpolation} in data, we fix $n=128$ and the sample size, with $m=172$ for $r=2.5$ and $m=57$ for $r=3.5$ (both near $\tau=2$). We then vary the number $k$ of perturbed edges while keeping $\norm{P-Q}{r}=0.75$, so that fewer edges receive larger shifts. We use both positive shifts and shifts with independent random signs. The two $r=3.5$ cells with $k=2$ are infeasible, because the shift would exceed $1/2$, and are omitted rather than clipped. Figure~\ref{fig:branches} compares, at level $0.05$, the $L_2$ statistic alone, the higher-order statistic alone, and their union.

\begin{figure}[ht]
\centering\includegraphics[width=\linewidth]{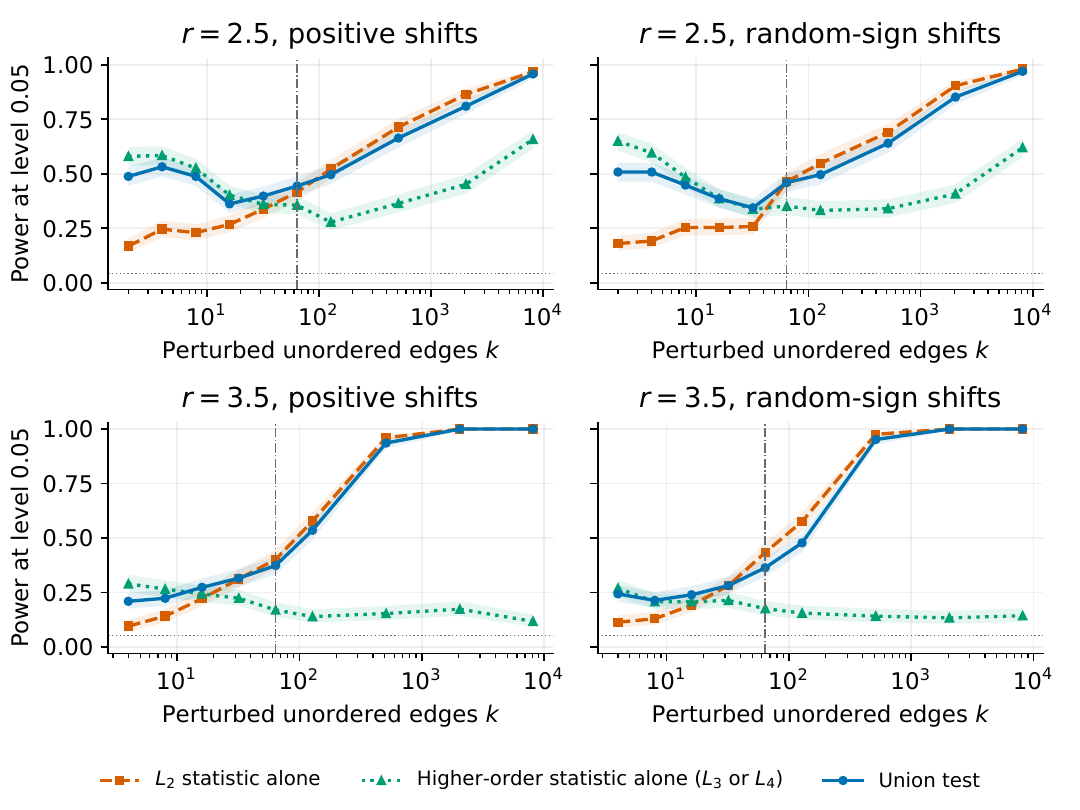}
\caption{Power at level $0.05$ (oracle calibration) against the number $k$ of perturbed edges, at $n=128$, $\epsilon=0.75$, and $m=172$ ($r=2.5$) or $m=57$ ($r=3.5$). The dash-dotted vertical line marks the hand-over point $k=\sqrt N/2$ of Figure~\ref{fig:interpolation}, and the dotted horizontal line is the null rejection rate of the union. Each alternative cell has 500 replicates and the null cell has 1,000. Bands are pointwise 95\% Wilson intervals conditional on the calibration sample.}
\label{fig:branches}
\end{figure}

The higher-order statistic is more powerful when only a few edges change, and the $L_2$ statistic when many edges change. The crossover lies between $k=32$ and $k=64$ for $r=2.5$ and between $k=16$ and $k=32$ for $r=3.5$, for both kinds of shift, close to the hand-over point $k\approx64$ of Figure~\ref{fig:interpolation}. The union follows the better statistic, up to the price of splitting the level between the two. For $r=2.5$ and positive shifts on $k=2$ edges, for example, the union has power $0.49$, against $0.58$ for the $L_3$ statistic alone and $0.17$ for the $L_2$ statistic alone. At the theorem thresholds, by contrast, the union rejects the null in $43.1\%$ ($r=2.5$) and $50.7\%$ ($r=3.5$) of replicates at these sample sizes. The same comparison at the theorem thresholds (Appendix~\ref{app:results}) therefore mostly reflects differences in null rejection.

\subsection{Further checks}
Appendix~\ref{app:results} reports further checks. They include three structured null models (Erd\H{o}s--R\'enyi, planted bisection, and a logistic $\beta$-model, with 68 cells at $n=50$ and $m=48$). They also include the practical normal calibration, a version of the union that splits the observations between the two statistics, and numerical evaluations of the lower bounds.

\section{Related work and limitations}\label{sec:limitations}
\citet{ghoshdastidar2020} established sparsity-sensitive minimax results for two-sample IER testing in the Frobenius and operator norms. \citet{dan2020} studied the related goodness-of-fit problem, in which the reference matrix is known. The integer statistics, their practical standardization, the result for real $r<2$, and both lower-bound constructions used here come from \citet{chatterjee2023}.

Rates for testing a known null are closely related. \citet{chhor2022} established local goodness-of-fit rates for multivariate binomial observations for $r\in[1,2]$. Their rates depend on the specified null vector, and they do not cover non-integral $r>2$. In the Gaussian sequence model the picture is older: Lemma~10 of \citet{kania2026} gives a simple-null separation of order $\sigma D^{1/(2r)}$ for fixed non-even $r>2$ and credits the rate to \citet{ingster2001}. Heuristically, setting $D\asymp n^2$ and $\sigma\asymp m^{-1/2}$ gives the same exponent as Theorem~\ref{thm:main}. For discrete distributions, \citet{waggoner2015} studied $L_p$ testing and \citet{balakrishnan2019} derived sharp local minimax rates for high-dimensional multinomials. \citet{jia2025} study goodness-of-fit testing in the Gaussian sequence model when the mean lies in an orthosymmetric set, such as an $\ell_p$ body. There the $\ell_p$ norm constrains the parameter and the separation is measured in $\ell_2$, whereas here the $L_r$ norm defines the separation itself.

A randomization argument shows that the unknown common null $P=Q$ is not what separates our problem from testing a known null. For each graph index $t$ and edge $e$, take $G_{t,e}$ or $1-H_{t,e}$ with probability one half each. The result is a Bernoulli variable with mean $1/2+(p_e-q_e)/2$. Flipping it independently with probability $1/4$ changes the mean to $1/2+(p_e-q_e)/4\in[1/4,3/4]$. This gives a problem with the known null mean $1/2$ and separation $\epsilon/4$, using no additional observations. The reduction does not by itself provide the goodness-of-fit bound for $r>2$ that would then be needed. Our proof instead works directly with the IER statistics, without a Gaussian approximation or any growth condition beyond $A_r\to\infty$.

\paragraph{A feasible calibration that is valid in finite samples.}
The theorem thresholds depend on $\epsilon$, and the oracle calibration needs the null. A calibration that is valid in finite samples is nevertheless available. Under $H_0$, each pair $(G_t,H_t)$ is exchangeable, meaning that swapping $G_t$ and $H_t$ does not change their joint distribution. The differences $D_t=G_t-H_t$ therefore keep their joint distribution when each one is multiplied by an independent random sign. Recomputing $U_2$ and $U_d$ on many sign-flipped copies of the data, and using the plus-one p-value of Appendix~\ref{app:oracle}, gives valid p-values for the unknown common null for any $n$ and $m$. The Bonferroni union of these p-values is a valid level-$\alpha$ test. We have not used this calibration in the reported experiments.

\paragraph{Limitations.}
Our results assume independent edges, aligned vertices, and a single edge-probability matrix within each group. They do not cover dependent edges, unmatched vertices, or latent graphon sampling. The theorem concerns a fixed finite $r$, and its constants may deteriorate as $r$ grows. The theorem thresholds are tuned to the sum of the two errors, so at the simulated sample sizes they can reject the null much more often than $5\%$. The normal calibration needs further assumptions that fail for sparse backgrounds. The simulations use four graph sizes and a finite number of replicates, so they illustrate the rate but cannot pin down an asymptotic exponent. The structured models serve as sanity checks.

\section{Conclusion}
We showed that a union of two published integer-norm tests, with thresholds balanced by interpolation, attains the conjectured sample complexity $n^{2/r}/\epsilon^2$ for non-integral $r>2$ in two-sample IER testing. The interpolation step removes the loss that comes from embedding every alternative into the next integer norm. The two tests hand over at a support of about $n/2$ edges. The sparse lower-bound construction places its perturbation at the same support size, suggesting that alternatives on about $n/2$ edges are the hard case for every non-integral $r>2$. Our simulations agree with the predicted scaling once all tests are compared at the same level.

\clearpage
\appendix
\section{Proof of Theorem~\ref{thm:main}}\label{app:proofs}
All norms below are unnormalized entrywise matrix norms. As in the main text, $M=n(n-1)/2$, $N=2M$, $h_e=p_e-q_e$, and $\|P-Q\|_r^r=2\sum_{e=1}^M|h_e|^r$. For a test $\phi$, the worst-case risk is
\[
 R(\phi)=\sup_{P=Q}\Prob_{P,Q}(\phi=1)
 +\sup_{\|P-Q\|_r\geq\epsilon}\Prob_{P,Q}(\phi=0),
 \qquad R^*=\inf_\phi R(\phi).
\]
The test that always accepts has $R=1$ whenever the alternative is nonempty, so $R^*\leq1$. An individual test can have a risk larger than one.

\paragraph{Moments of the integer statistic.}
Fix an integer $s\geq2$ and $m\geq s$, put $K=\lfloor m/s\rfloor$, and use $s$ disjoint blocks of $K$ graph pairs. The block means $Y_{\ell,e}$ have mean $h_e$ and variance
\[
 v_e=\{p_e(1-p_e)+q_e(1-q_e)\}/K\leq (2K)^{-1}.
\]
For even $s$ put $W_e=\prod_{\ell=1}^sY_{\ell,e}$, and for odd $s$ put $W_e=(\prod_{\ell=1}^{s-1}Y_{\ell,e})|Y_{s,e}|$. These are the block-sum statistics of \citet{chatterjee2023} divided by $K^s$. Because the blocks are independent,
\[
 \mathbb EW_e=h_e^s\quad(s\text{ even}),\qquad
 \mathbb EW_e=h_e^{s-1}\mathbb E|Y_{s,e}|\geq |h_e|^s
 \quad(s\text{ odd}),
\]
where the inequality uses Jensen's inequality, $\E|Y_{s,e}|\geq|h_e|$, and the fact that $s-1$ is even. For both parities the expectation is zero when $h_e=0$, and $\mathbb EW_e^2=(h_e^2+v_e)^s$. In both cases $(\E W_e)^2\geq h_e^{2s}$. For odd $s$ this follows from $(\E W_e)^2=h_e^{2(s-1)}(\E|Y_{s,e}|)^2$ and the same Jensen bound. The edge terms are independent, so $U_s=\sum_eW_e$ satisfies
\begin{align}
 \operatorname{Var}(U_s)
 &\leq\sum_e\{(h_e^2+v_e)^s-|h_e|^{2s}\}\\
 &\leq\sum_{j=1}^s\binom{s}{j}(2K)^{-j}
       \sum_e |h_e|^{2s-2j}.\label{eq:variance-bound}
\end{align}
Subtracting $h_e^{2s}$ is what removes the $j=0$ term, which would not become small. Let $S=\sum_e|h_e|^s>0$ and $q=2s-2j$. Monotonicity of $\ell_p$ norms when $q\geq s$, and H\"older's inequality when $0<q<s$, give
\[
 \sum_e|h_e|^q\leq M^{\max(1-q/s,0)}S^{q/s}.
\]
When $q=0$, the left side is exactly $M$, since it counts every edge, including those with $h_e=0$, and the inequality holds with equality.

\paragraph{Uniform error bounds.}
The test rejects when $U_s\geq\eta^s/4$ for some $\eta>0$. Under the null, $h_e=0$ for every edge, so $\E U_s=0$ and \eqref{eq:variance-bound} reduces to $\operatorname{Var}(U_s)\leq M(2K)^{-s}$. Chebyshev's inequality then gives
\[
 \sup_{P=Q}\Prob(U_s\geq\eta^s/4)
 \leq16M(2K)^{-s}\eta^{-2s}.
\]
Under an alternative with $\|P-Q\|_s\geq\eta$, we have $S=\|P-Q\|_s^s/2\geq\eta^s/2$. Since $\E U_s\geq S$, also $\E U_s-\eta^s/4\geq S/2$, so Chebyshev's inequality bounds the type~II error by $4\operatorname{Var}(U_s)/S^2$. By \eqref{eq:variance-bound} and the $\ell_p$ bound above, this is at most
\begin{align*}
 4\sum_{j=1}^s\binom{s}{j}(2K)^{-j}
 M^{\max(2j/s-1,0)}S^{-2j/s}
 &\leq4\sum_{j=1}^s\binom{s}{j}2^{2j/s-j}
 \left(\frac{M^{1/s}}{K\eta^2}\right)^j,
\end{align*}
where we used $S^{-2j/s}\leq2^{2j/s}\eta^{-2j}$, $\max(2j/s-1,0)\leq j/s$ for $1\leq j\leq s$, and $M\geq1$. The type~I bound equals $16\cdot2^{-s}t^s$ with $t$ below. Consequently, with the finite constant $C_s=16\cdot2^{-s}+4\max_{1\leq j\leq s}\binom sj2^{2j/s-j}$, which depends only on $s$,
\begin{equation}
 R(\psi_{s,\eta})\leq C_s\sum_{j=1}^s t^j,
 \qquad t=\frac{M^{1/s}}{K\eta^2}.\label{eq:uniform-integer}
\end{equation}
The bounds hold simultaneously over the whole null and alternative classes, and they allow the matrices and the separation to change with $n$. Alternatives exactly on the boundary are included.

\paragraph{Interpolation and sample costs.}
For $1\leq r<2$, H\"older's inequality gives $\|x\|_2\geq\epsilon N^{1/2-1/r}$ whenever $\|x\|_r\geq\epsilon$. The $s=2$ bound therefore costs $M^{1/2}N^{2/r-1}/\epsilon^2\asymp n^{4/r-1}/\epsilon^2$. For non-integral $r>2$, let $d=\lceil r\rceil$ and
\[
 \alpha=\frac{1/r-1/d}{1/2-1/d},\qquad
 \eta_2=\epsilon N^{1/4-1/(2r)},\qquad
 \eta_d=\epsilon N^{1/(2d)-1/(2r)}.
\]
H\"older's inequality applied to $|x_i|^{\alpha r}|x_i|^{(1-\alpha)r}$ with the conjugate exponents $2/(\alpha r)$ and $d/((1-\alpha)r)$ proves $\|x\|_r\leq\|x\|_2^\alpha\|x\|_d^{1-\alpha}$. The identity $\alpha/2+(1-\alpha)/d=1/r$ gives $\eta_2^\alpha\eta_d^{1-\alpha}=\epsilon$. If both branch norms were strictly below their thresholds, their weighted product would be strictly below $\epsilon$. Every alternative therefore belongs to at least one branch alternative, including the equality cases. Both branch costs satisfy
\[
 \frac{M^{1/2}}{\eta_2^2}=2^{-1/2}\frac{N^{1/r}}{\epsilon^2},
 \qquad
 \frac{M^{1/d}}{\eta_d^2}=2^{-1/d}\frac{N^{1/r}}{\epsilon^2}.
\]
Since $K_s=\lfloor m/s\rfloor\geq m/(2s)$ whenever $m\geq s$, and $N\leq n^2$, each branch has $t_s=M^{1/s}/(K_s\eta_s^2)\leq2s\cdot2^{-1/s}N^{1/r}/(m\epsilon^2)\leq2sA_r/m$, where $A_r=n^{2/r}/\epsilon^2$. We compute both statistics on the same observations, and each keeps its own independent blocks. The type~I error of the union is at most the sum of the branch type~I errors, and each alternative's acceptance probability is at most that of a branch whose alternative contains it. Hence, by \eqref{eq:uniform-integer},
\[
 R(\phi_{r,\epsilon})\leq R(\psi_{2,\eta_2})+R(\psi_{d,\eta_d})
 \leq C_2\sum_{j=1}^2\Big(\frac{4A_r}{m}\Big)^j+C_d\sum_{j=1}^d\Big(\frac{2dA_r}{m}\Big)^j,
\]
which tends to zero when $m/A_r\to\infty$, and no independence between the branches is used. In the same way, for $1\leq r<2$ the $s=2$ test has $t\leq4A_r/m$ with $A_r=n^{4/r-1}/\epsilon^2$, because $N^{2/r-1/2}\leq n^{4/r-1}$. For integral $r\geq2$ the single statistic of order $r$ with $\eta=\epsilon$ has $t\leq2rA_r/m$ with $A_r=n^{2/r}/\epsilon^2$, and $r=2$ gives the rate $n/\epsilon^2$. For comparison, the direct reduction to the ceiling norm has the sufficient cost $n^{4/r-2/d}/\epsilon^2$. This is only an upper bound on the sample size that this test needs.

\paragraph{Dense construction.}
Set $Q_e=1/2$ and $P_e=1/2+\gamma_e\delta$ with independent uniform signs $\gamma_e$ and $\delta=\epsilon N^{-1/r}$, and assume $\delta\leq1/2$. Every draw has norm exactly $\epsilon$. Write $A_{te}$ for the observation of edge $e$ in the $t$-th graph of the first group. Under the null these are independent Bernoulli$(1/2)$ variables. The second group has the same distribution under both hypotheses, so it contributes a likelihood ratio of one. For a fixed sign $\gamma\in\{-1,1\}$, edge $e$ contributes
\[
 \ell_{\gamma,e}=\prod_{t=1}^m[1+2\gamma\delta(2A_{te}-1)],
\]
and the likelihood ratio of the random construction is $L=\prod_e\tfrac12(\ell_{+1,e}+\ell_{-1,e})$. Because $\E_0(2A_{te}-1)=0$ and $(2A_{te}-1)^2=1$, independence gives $\mathbb E_0\ell_{\gamma,e}\ell_{\gamma',e}=(1+4\gamma\gamma'\delta^2)^m$. Averaging over the four sign pairs and multiplying over the $M$ independent edges gives
\[
 \mathbb E_0L^2=\left\{\frac{(1+4\delta^2)^m+(1-4\delta^2)^m}{2}\right\}^M.
\]
For $x\geq0$, the average $\{(1+x)^m+(1-x)^m\}/2$ keeps only the even terms of the binomial expansion, so it is at most $\sum_{j\geq0}(mx)^{2j}/(2j)!=\cosh(mx)\leq\exp(m^2x^2/2)$, where the last inequality follows by integrating $\tanh u\leq u$ for $u\geq0$. With $x=4\delta^2$ this gives $\log\mathbb E_0L^2\leq8Mm^2\delta^4$. For $1\leq r\leq2$ and $A_r=n^{4/r-1}/\epsilon^2$,
\[
 \delta^2=\frac{1}{nA_r}\left(\frac n{n-1}\right)^{2/r},\qquad
 8Mm^2\delta^4=4\Big(\frac n{n-1}\Big)^{4/r-1}\Big(\frac m{A_r}\Big)^2\leq D_r(m/A_r)^2,
\]
with $D_r=2^{1+4/r}$, for all $n\geq2$.

\paragraph{Sparse construction.}
For $r\geq2$, choose a uniformly random support of $k=\lfloor n/2\rfloor$ of the $M$ unordered edges, and set $P_e=1/2+\delta$ on the support and $P_e=1/2$ elsewhere, where $\delta=\epsilon(2k)^{-1/r}\leq1/2$. Again every draw has norm exactly $\epsilon$. The likelihood ratio is $L=\E_S\prod_{e\in S}\ell_{+1,e}$, an average over the random support $S$. Since $\E_0\ell_{+1,e}=1$ and $\E_0\ell_{+1,e}^2=(1+4\delta^2)^m$, two independent supports $S$ and $S'$ that share $Z$ edges give
\[
 \mathbb E_0L^2=\mathbb E a^Z,\qquad a=(1+4\delta^2)^m\geq1.
\]
For $0\leq j\leq k$, counting the $j$-subsets that lie in both supports gives
\[
 \mathbb E\binom Zj=\frac{\binom kj^2}{\binom Mj}
 =\binom kj\frac{(k)_j}{(M)_j}\leq\binom kj(k/M)^j,
\]
where $(b)_j$ is the falling factorial and each ratio $(k-i)/(M-i)$ is at most $k/M$. Expanding $a^Z=(1+(a-1))^Z=\sum_j\binom Zj(a-1)^j$ and using $1+y\leq e^y$ then proves, without any concentration argument,
\[
 \mathbb E a^Z\leq[1+(k/M)(a-1)]^k
 \leq\exp\{(k^2/M)(a-1)\}.
\]
For $A_r=n^{2/r}/\epsilon^2$ we have $\delta^2=A_r^{-1}(n/(2k))^{2/r}\leq B_r/A_r$ with $B_r=2^{2/r}$ for all $n\geq2$, and $k^2/M\leq1$. Hence $\log\E_0L^2\leq(k^2/M)(a-1)\leq a-1\leq e^{4m\delta^2}-1$, and keeping this nonlinear bound gives
\begin{equation}
 \log\mathbb E_0L^2\leq\exp(4B_rm/A_r)-1,\label{eq:sparse-nonlinear}
\end{equation}
which is useful even when $m/A_r$ is a fixed positive constant.

\paragraph{Quantifiers and fixed error.}
Consider any sequence $n\to\infty$ and $\epsilon_n>0$ with $A_r=n^{\max(4/r-1,2/r)}/\epsilon_n^2\to\infty$, for a fixed finite $r\geq1$. The formulas for $\delta^2$ above show that the relevant construction is valid, that is, $\delta\leq1/2$, for all large $n$, so the alternative is nonempty there. Write $P_0$ for the distribution of the data under the null and $P_\pi$ for its distribution under the random construction. The worst-case risk of every test $\phi$ is at least $P_0(\phi=1)+P_\pi(\phi=0)\geq1-\operatorname{TV}(P_0,P_\pi)$, and the Cauchy--Schwarz inequality gives $\operatorname{TV}(P_0,P_\pi)=\tfrac12\E_0|L-1|\leq\tfrac12\sqrt{\E_0L^2-1}$. Hence
\[
 R^*\geq1-\operatorname{TV}(P_0,P_\pi)
 \geq1-\tfrac12\sqrt{\mathbb E_0L^2-1}.
\]
Together with $R^*\leq1$, either construction proves $R^*\to1$ whenever $m/A_r\to0$. The upper bounds prove $R^*\to0$ whenever $m/A_r\to\infty$. For a fixed $\beta\in(0,1)$, choose $c_{r,\beta}>0$ small enough that
\[
 \tfrac12\sqrt{e^{D_rc_{r,\beta}^2}-1}<1-\beta
 \quad(r\leq2),\qquad
 \tfrac12\sqrt{\exp\{e^{4B_rc_{r,\beta}}-1\}-1}<1-\beta
 \quad(r\geq2).
\]
Both second-moment bounds increase with $m$, so every integer $m\leq c_{r,\beta}A_r$ has $R^*>\beta$ for all large $n$. Conversely, the bounds $t_s\leq2sA_r/m$ above show that the polynomial risk bounds of the one or two integer tests sum to at most $\beta$ at $m=\lceil C_{r,\beta}A_r\rceil$ once $C_{r,\beta}$ is large enough. Because $A_r\to\infty$, these sample sizes eventually meet every fixed block-count requirement. Hence $c_{r,\beta}A_r\leq m^*_\beta\leq(C_{r,\beta}+1)A_r$ for all large $n$, where $m^*_\beta=\min\{m\geq1:R^*\leq\beta\}$. The argument is for a fixed finite $r$ and does not track sharp constants.

\section{Simulation details}\label{app:experiments}
\FloatBarrier
\subsection{Exact block-count simulation}
For a segment of $c$ observations and an unordered edge $e$, the simulator draws independent counts $X_e\sim\operatorname{Binomial}(c,p_e)$ and $Z_e\sim\operatorname{Binomial}(c,q_e)$. Counts from disjoint segments and different edges are independent. Adding segment counts gives exactly the block counts that each statistic needs, and dividing by the block size gives its block means. Because all statistics aggregate the same segment counts, their joint distribution is the same as with explicit graph sampling. The reduction is exact under the IER model.

The implementation stores only unordered edges. It computes the matrix norm with the factor two in~\eqref{eq:norm} and sums each statistic once per unordered edge. Before a simulation starts, it checks that all probabilities lie in $[0,1]$ and that the achieved norm equals the target, and it never runs a statistic with fewer graphs than it has blocks. Counts and replicates can be processed in chunks without changing the model. Pilot runs, calibration, and evaluation use separate random seeds.

\FloatBarrier
\subsection{Practical variance estimator}
For an integer $s\geq2$ and $m'=2sC$ graphs per group, let $K=2C$. Split the first group into $2s$ consecutive blocks of $C$ graphs and define
\begin{equation}\label{eq:phat}
 \widehat p_{\ell,e}=\frac1C\sum_{t=(\ell-1)C+1}^{\ell C}G_{t,e},\qquad
 \widehat V_s=\left(\frac2K\right)^s\sum_{e\in E_n}
 \left\{\prod_{\ell=1}^s\widehat p_{\ell,e}
 \prod_{\ell=s+1}^{2s}(1-\widehat p_{\ell,e})\right\}.
\end{equation}
This is Eq.~(13) of \citet{chatterjee2023}, rescaled from their block-sum statistic to $U_s$. Under $P=Q$ it is a nonnegative, unbiased estimate of $\operatorname{Var}(U_s)$. The numerator uses $s$ blocks of size $K$ on the same $m'$ graphs. When $m$ is not divisible by $2s$, the implementation uses $m'=2s\lfloor m/(2s)\rfloor$ for both the numerator and the variance estimate.

The normal approximation of \citet{chatterjee2023} requires a fixed $s$, $\norm{P^{(n)}}{s}\to\infty$, and $\sup_n\norm{P^{(n)}}{\infty}<1$. It leads to the two-sided rule
\[
 |U_s|/\sqrt{\widehat V_s}>\Phi^{-1}(1-\alpha_{\rm test}/2).
\]
For the union of two statistics we use a Bonferroni split, replacing $\alpha_{\rm test}$ by $\alpha_{\rm test}/2$ in each, which gives the critical value $\Phi^{-1}(1-\alpha_{\rm test}/4)$. The Bonferroni inequality holds even though the two statistics are dependent. When $\widehat V_s=0$ the standardized statistic is undefined, and we record the decision as undefined rather than as an acceptance. A union that involves an undefined statistic is also undefined. We report rejection rates among the defined decisions together with the fraction of undefined ones, so these rates are conditional on a defined standardization.

For integer $r>2$, the practical and oracle unions combine orders $2$ and $r$, whereas the theorem uses order $r$ alone. These unions are therefore additional comparisons, and the individual order-$r$ results give the direct calibration comparison.

The normal approximation can fail for sparse graphs. For a planted bisection with probabilities proportional to $1/n$, $\norm{P}{s}^s=O(n^{2-s})$, which stays bounded for $s\geq2$. The conditions above therefore fail, even when a finite-$n$ histogram looks roughly normal.

\FloatBarrier
\subsection{Alternatives and structured models}\label{app:structured}
For the dense construction, $Q_e=1/2$ and $P_e=1/2+\gamma_e\epsilon N^{-1/r}$. For a hidden support of size $k$, we perturb $k$ uniformly chosen unordered edges by $\epsilon(2k)^{-1/r}$, with positive or independent random signs as two separate families. Each construction has exactly the requested norm when it is feasible, that is, when all probabilities stay in $[0,1]$. The support study changes $k$ with the norm held fixed, so the shift per edge grows as $k$ shrinks.

The structured models are defined by their probability matrices. Erd\H{o}s--R\'enyi graphs have a common off-diagonal probability. A planted bisection has two fixed communities of equal size, with within-community probability $a$ and between-community probability $b$. The logistic $\beta$-model has $p_{ij}=\operatorname{logistic}(\beta_i+\beta_j)$ for $i\neq j$. The $\beta$ vector and the community assignment stay fixed across the graphs of a replicate. A perturbation of the latent $\beta$ vector does not directly fix the norm of the induced probability matrix, so we calibrate the induced matrices to the target norm.

The structured comparisons use $n=50$, $m=48$, $r\in\{2,2.5,3.5,4\}$, and $\epsilon\in\{0.25,0.5,1\}$. The Erd\H{o}s--R\'enyi null has $q_e=0.3$. The bisection null has within-community probability $0.3$ and between-community probability $0.1$, and the alternative raises the first and lowers the second by the same amount. The $\beta$-model fixes $\beta_i$ on an equally spaced grid from $-1$ to $1$, sets $q_{ij}=\operatorname{logistic}(\beta_i+\beta_j)$, and solves for a common latent shift that yields the requested matrix norm. Two bisection alternatives are infeasible, and we omit them together with their paired null cells, which leaves 68 cells, each with 500 evaluation and 999 calibration replicates. These settings adapt the examples of \citet{chatterjee2023} rather than reproduce their Figure~2.

We also reproduce, with fewer repetitions, their fourth-norm null experiment (their Figure~1) at $n=100$, $m=32$, and $s=4$. The Erd\H{o}s--R\'enyi probability is $1/2$, and the bisection probabilities are $8/n$ and $2/n$. The $\beta$ vector is drawn uniformly from the unit sphere once per replicate and kept fixed across that replicate's graphs. We use 200 repetitions per model and follow the definitions in the published paper.

\FloatBarrier
\subsection{Oracle calibration}\label{app:oracle}
For each order $s$, the oracle calibrates the raw statistic $T=U_s$, with $s$ consecutive blocks of size $\lfloor m/s\rfloor$ as in Section~\ref{sec:construction}. It rejects in the upper tail, whereas the practical normal calibration is two-sided and uses an estimated standard deviation. A single statistic uses level $0.05$, and the union gives level $0.05/J$ to each of its $J$ orders, $\{2,\lceil r\rceil\}$ without duplicates. With $B$ independent calibration statistics $T_1,\ldots,T_B$ simulated under the null, an evaluation statistic $T$ receives the Monte Carlo p-value
\[
 p_{\rm MC}=\frac{1+\sum_{b=1}^B\ind\{T_b\geq T\}}{B+1}.
\]
The statistics of a union share their calibration samples, which Bonferroni allows, and calibration and evaluation samples are independent. Exchangeability makes this p-value valid under the specified null, including ties. For one realized calibration sample, the conditional rejection probability can differ from the nominal level.

Let $\alpha_{\rm branch}$ be the level given to one statistic, that is, $0.05$ or $0.05/J$, and let $q=\lfloor\alpha_{\rm branch}(B+1)\rfloor$. When $q=0$ the rule never rejects. For $q\geq1$ it rejects strictly above the $(B-q+1)$st smallest calibration statistic. The Dvoretzky--Kiefer--Wolfowitz inequality then shows that, with probability at least $0.95$ over the calibration sample, the conditional null rejection probability of one statistic is at most
\[
 (q-1)/B+\sqrt{\log(40)/(2B)}.
\]
This bound allows atoms but is conservative: for $B=999$ and level $0.025$ it is about $0.067$. It holds for one statistic at a time, not simultaneously over statistics or cells. Table~\ref{tab:calsens} adds a bootstrap sensitivity analysis that resamples the calibration and evaluation replicates independently at four prespecified settings.

\begin{table}[ht]\centering
\caption{Sensitivity of the oracle union to the calibration sample at $n=256$ and $\epsilon=0.75$. Each range is the 2.5th--97.5th percentile over 400 bootstrap resamples of independent calibration and evaluation replicates, keeping the two statistics paired. The ranges describe variability and are not exact confidence intervals.}\label{tab:calsens}
\begin{tabular}{rrlc}\toprule $r$ & $m$ & Model & Rejection-rate range\\\midrule
2.5 & 300 & null & [0.025, 0.060] \\
2.5 & 300 & sparse & [0.390, 0.516] \\
2.5 & 1201 & null & [0.031, 0.076] \\
2.5 & 1201 & sparse & [1.000, 1.000] \\
3.5 & 85 & null & [0.036, 0.074] \\
3.5 & 85 & sparse & [0.280, 0.452] \\
3.5 & 338 & null & [0.048, 0.085] \\
3.5 & 338 & sparse & [1.000, 1.000] \\
\bottomrule\end{tabular}\end{table}

\FloatBarrier
\subsection{Uncertainty and reproducibility}
A rejection count $x$ out of $B$ independent evaluation replicates estimates a probability by $\widehat p=x/B$. Pointwise Wilson intervals use $z=\Phi^{-1}(0.975)$, with center and half-width
\[
 \frac{\widehat p+z^2/(2B)}{1+z^2/B},\qquad
 \frac{z}{1+z^2/B}\sqrt{\frac{\widehat p(1-\widehat p)}B+\frac{z^2}{4B^2}}.
\]
These intervals describe binomial Monte Carlo error. When two methods are applied to the same data, we compare them through their paired rejection indicators rather than as independent samples. Oracle calibration adds quantile uncertainty that the evaluation intervals do not include. A transition point is reported only when the grid brackets it, and it is marked as censored otherwise.

The main and structured studies use 212,000 evaluation and 391,608 calibration replicates. All simulations ran on the CPU of a single Apple Silicon Mac with 16~GB of memory, using one BLAS thread per worker. The recorded run times of all cells, including calibration, sum to about 45 minutes. For every cell we saved the configuration, the random seeds, and the per-replicate statistics and decisions, together with the figure scripts and pinned dependency versions. Small exact checks of likelihood ratios and statistics supplement the proofs but do not replace them.

\section{Additional results}\label{app:results}
\FloatBarrier
\subsection{The theorem test at its own thresholds}
Figure~\ref{fig:theorem-scaling} repeats Figure~\ref{fig:scaling} for the theorem test, whose thresholds depend on $\epsilon$. Its null rejection is high for small $\tau$ and falls to nearly zero for $\tau\geq8$, while its power rises to one. Figure~\ref{fig:theorem-scaling-m} shows the same data against $m$ instead of $\tau$, and Figure~\ref{fig:errsum} shows the sum of the null rejection rate and the miss rate.

\begin{figure}[ht]
\centering\includegraphics[width=\linewidth]{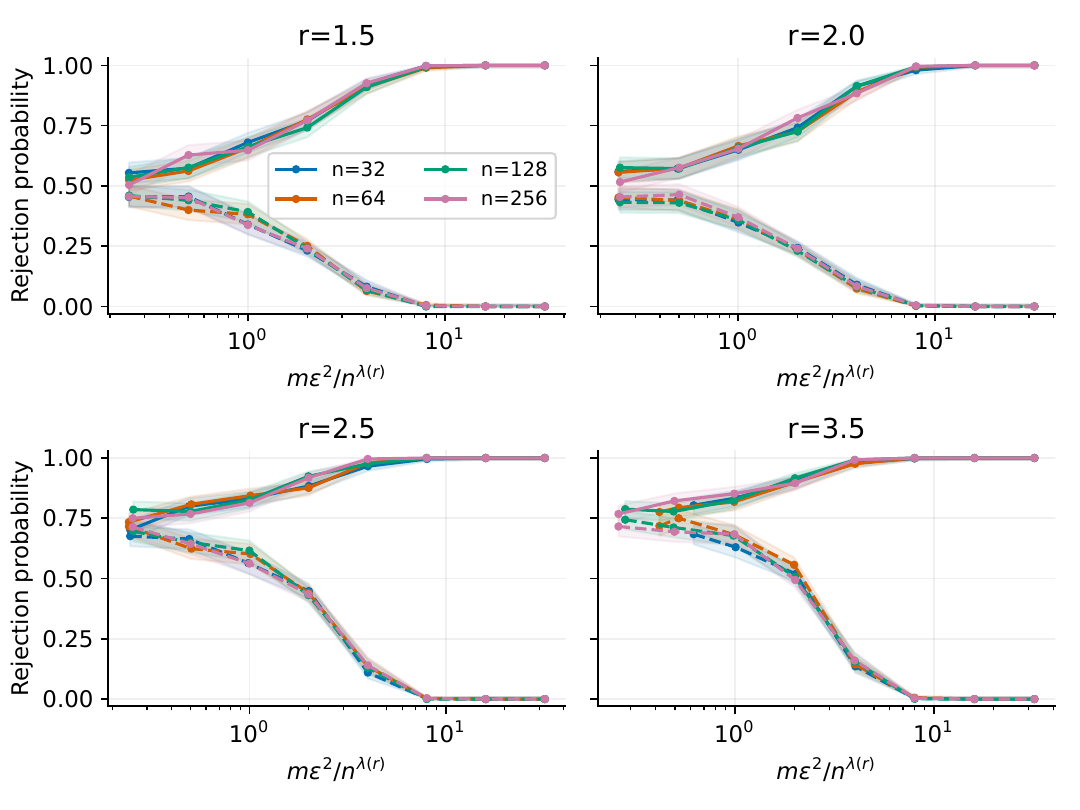}
\caption{Theorem test at its own thresholds, with $\epsilon=0.75$. Solid lines show power under the dense construction ($r\leq2$) or the hidden support of $k=\lfloor n/2\rfloor$ edges ($r>2$), averaged over the construction. Dashed lines show the null $P=Q=1/2$. Colors identify $n$. The horizontal axis uses the actual rounded $m$, and bands are pointwise 95\% Wilson intervals over 500 replicates (1,000 for the null at the grid points nearest $\tau=2$ and $8$).}
\label{fig:theorem-scaling}
\end{figure}

\begin{figure}[ht]
\centering\includegraphics[width=\linewidth]{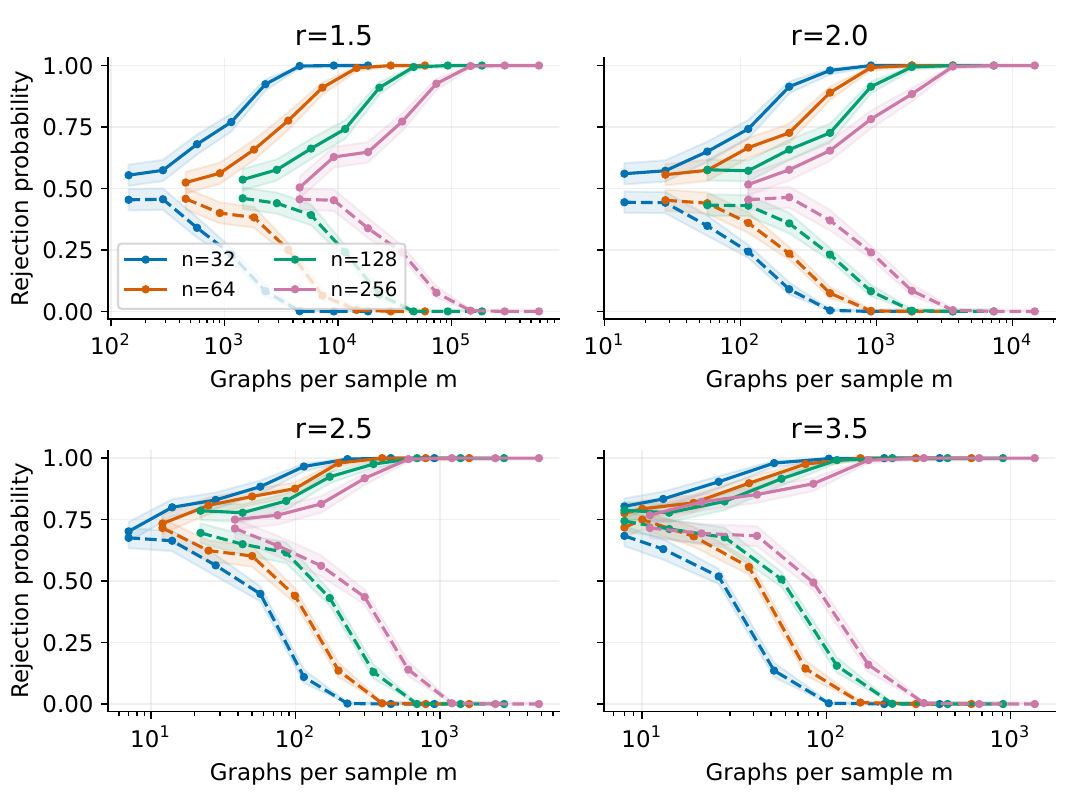}
\caption{The data of Figure~\ref{fig:theorem-scaling} plotted against the number of graphs per group $m$.}
\label{fig:theorem-scaling-m}
\end{figure}

\begin{figure}[ht]
\centering\includegraphics[width=\linewidth]{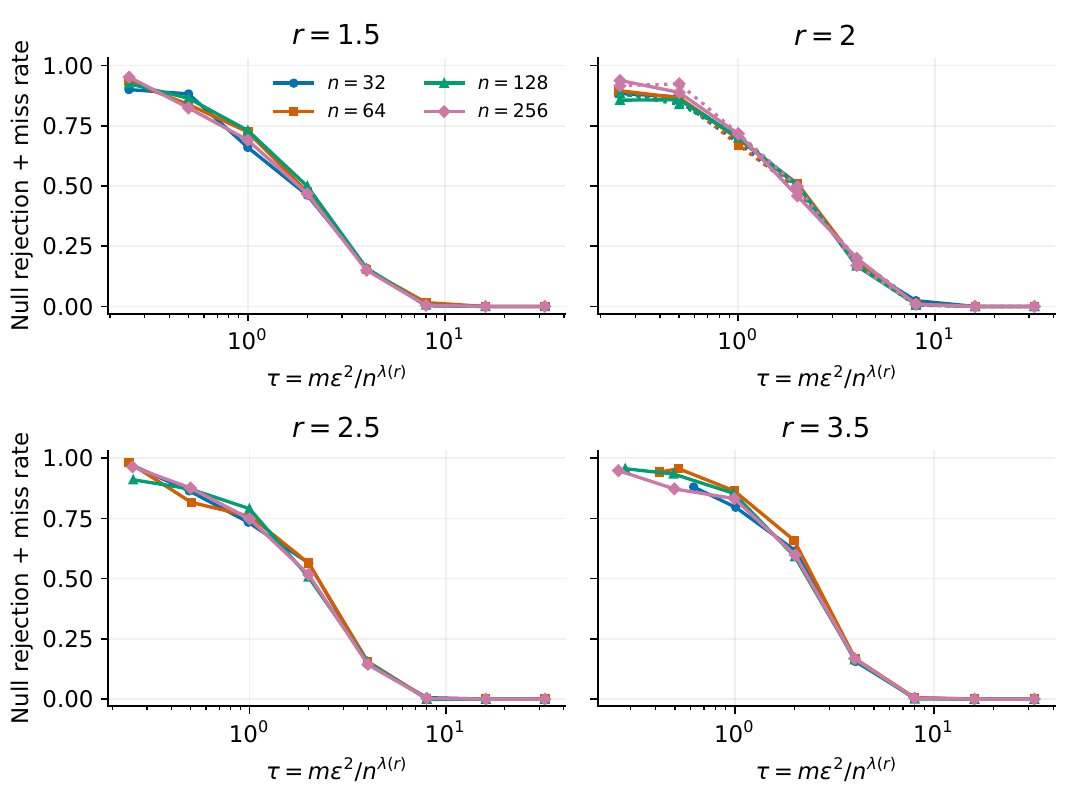}
\caption{Null rejection rate plus miss rate of the theorem test, for the cells of Figure~\ref{fig:theorem-scaling}. Colors and markers identify $n$, and at $r=2$ dotted lines also show the sparse construction. These sums refer to the two simulated models, whereas the worst-case risk~\eqref{eq:risk} takes a supremum over all models.}
\label{fig:errsum}
\end{figure}

\FloatBarrier
\subsection{Prespecified transition analysis}\label{app:transitions}
Before running the study we planned to track the power of the theorem test at its own thresholds and to locate where it first reaches $0.8$. After a weighted isotonic fit, the grid brackets this target in 19 of 20 size and construction series. We do not extrapolate, so the remaining series is censored. Each series uses 400 bootstrap repetitions of complete replicates, and we withhold the interval whenever any bootstrap crossing is censored. Table~\ref{tab:transitions} reports the fitted slopes. For $r=2.5$ and $3.5$, the power of the theorem test is already near $0.8$ at the smallest $\tau$, because its null rejection there is high (Figure~\ref{fig:theorem-scaling}). This explains the slope of $1.066$ at $r=2.5$ and the censoring at $r=3.5$. The equal-size analysis in Table~\ref{tab:slopes} avoids this problem.

\begin{table}[ht]\centering
\caption{Prespecified analysis: slope of $\log m^*$ against $\log n$ over $n=32,64,128,256$, where $m^*$ is the number of graphs at which the power of the theorem test reaches $0.8$ ($\epsilon=0.75$). A slope or interval is withheld when a crossing is censored.}\label{tab:transitions}
\begin{tabular}{rlccc}\toprule $r$ & Construction & Finite crossings & Slope & Bootstrap interval\\\midrule
1.5 & dense & 4/4 & 1.679 & [1.590, 1.774] \\
2 & dense & 4/4 & 0.945 & [0.846, 1.036] \\
2 & sparse & 4/4 & 1.014 & [0.930, 1.104] \\
2.5 & sparse & 4/4 & 1.066 & withheld \\
3.5 & sparse & 3/4 & censored & withheld \\\bottomrule\end{tabular}\end{table}

\FloatBarrier
\subsection{Support study at the theorem thresholds}
Figure~\ref{fig:theorem-support} repeats Figure~\ref{fig:branches} at the theorem thresholds and adds the direct reductions to $L_2$ (threshold $\eta=\epsilon$) and to $L_d$ (threshold $\eta=\epsilon N^{1/d-1/r}$). The labels ``Union L2 branch'' and ``Union high branch'' refer to the two statistics of the union at the union thresholds. Table~\ref{tab:support-null} gives the null rejection rates at these sample sizes, which differ widely between methods, and Figure~\ref{fig:decomposition} shows which statistic triggers each rejection of the union. The published sufficient sample size for the direct ceiling reduction is an upper bound, and it does not predict how that test's power changes with $k$.

\begin{figure}[ht]
\centering\includegraphics[width=\linewidth]{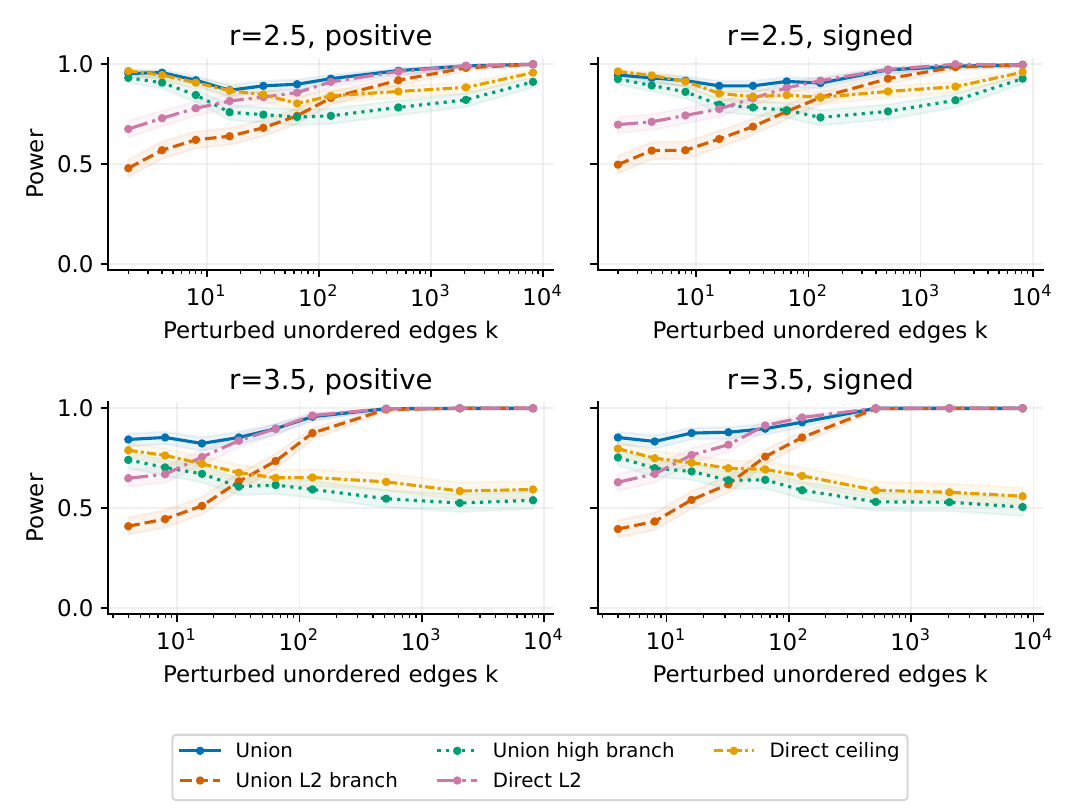}
\caption{Power at the theorem thresholds against the number $k$ of perturbed edges, at $n=128$, $\epsilon=0.75$, $m=172$ ($r=2.5$) and $m=57$ ($r=3.5$). The union rejects the null in $43.1\%$ and $50.7\%$ of replicates at these settings. Each cell has 500 replicates, and bands are pointwise 95\% Wilson intervals.}
\label{fig:theorem-support}
\end{figure}

\begin{table}[ht]
\centering
\caption{Null rejection rates at the sample sizes of the support study ($n=128$, $P=Q=1/2$, $\epsilon=0.75$). All methods in a row block use the same observations. Intervals are pointwise 95\% Wilson intervals.}\label{tab:support-null}
\begin{tabular}{rrlcc}\toprule
$r$ & $m$ & Method & Rejections & Interval\\\midrule
2.5 & 172 & Union & 431/1000 & [0.401, 0.462] \\
2.5 & 172 & Union, $L_2$ statistic & 230/1000 & [0.205, 0.257] \\
2.5 & 172 & Union, $L_3$ statistic & 293/1000 & [0.266, 0.322] \\
2.5 & 172 & Direct $L_2$ & 382/1000 & [0.352, 0.413] \\
2.5 & 172 & Direct ceiling & 424/1000 & [0.394, 0.455] \\
3.5 & 57 & Union & 507/1000 & [0.476, 0.538] \\
3.5 & 57 & Union, $L_2$ statistic & 222/1000 & [0.197, 0.249] \\
3.5 & 57 & Union, $L_4$ statistic & 370/1000 & [0.341, 0.400] \\
3.5 & 57 & Direct $L_2$ & 455/1000 & [0.424, 0.486] \\
3.5 & 57 & Direct ceiling & 435/1000 & [0.405, 0.466] \\
\bottomrule\end{tabular}
\end{table}

\begin{figure}[ht]
\centering\includegraphics[width=\linewidth]{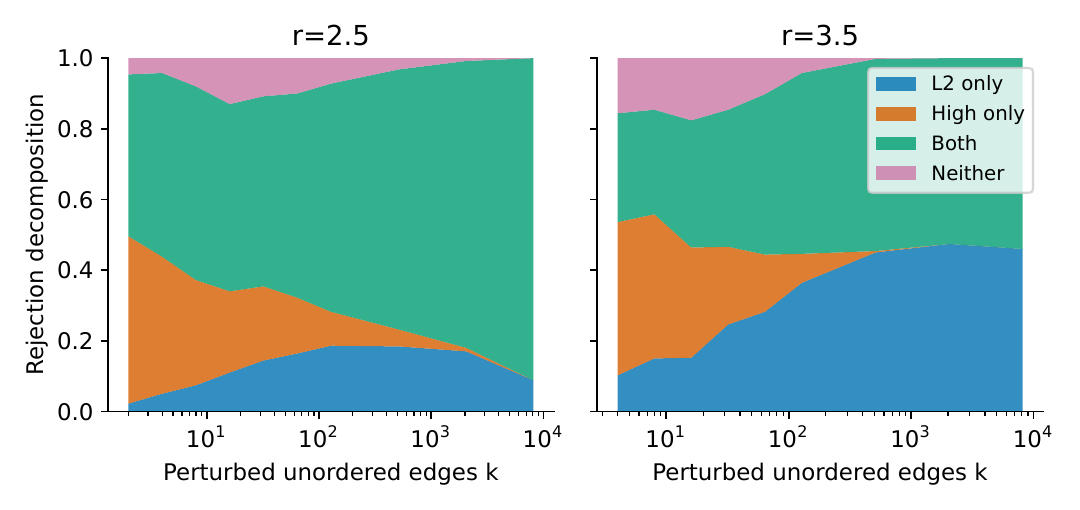}
\caption{Which statistic triggers each rejection of the union at the theorem thresholds, for positive shifts at $n=128$, $\epsilon=0.75$, and the sample sizes of Table~\ref{tab:support-null}. Colors show rejections by the $L_2$ statistic only, by the higher-order statistic only, by both, and by neither. Each support size has 500 replicates.}
\label{fig:decomposition}
\end{figure}

\FloatBarrier
\subsection{Splitting the observations between the two statistics}
Instead of computing both statistics on all $m$ graphs, one can give the first $\lfloor m/2\rfloor$ graphs to the $L_2$ statistic and the remaining $m-\lfloor m/2\rfloor$ graphs to the $L_d$ statistic. The thresholds stay the same, and each statistic uses its actual block count. The union on shared data gives each statistic all $m$ graphs, whereas the split version gives each statistic about half of them. Figure~\ref{fig:split} compares the two versions on the sparse construction. Because they use different amounts of data per statistic, their null rejection rates also differ.

\begin{figure}[ht]
\centering\includegraphics[width=\linewidth]{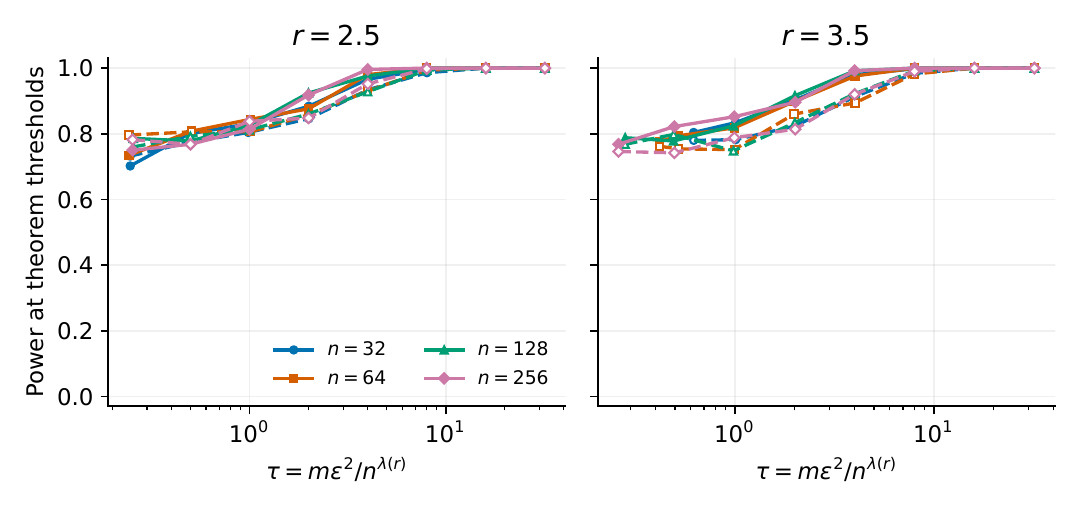}
\caption{Power at the theorem thresholds of the union on shared data (solid lines, filled markers) and on split data (dashed lines, open markers) under the sparse construction, with $\epsilon=0.75$ and $r=2.5,3.5$. Colors and markers identify $n$. Each cell has 500 replicates.}
\label{fig:split}
\end{figure}

\FloatBarrier
\subsection{Numerical lower bounds}
Figure~\ref{fig:lowerbounds} evaluates the lower bound of Section~\ref{sec:lower} numerically, using the exact second moments rather than their upper bounds.

\begin{figure}[ht]
\centering\includegraphics[width=\linewidth]{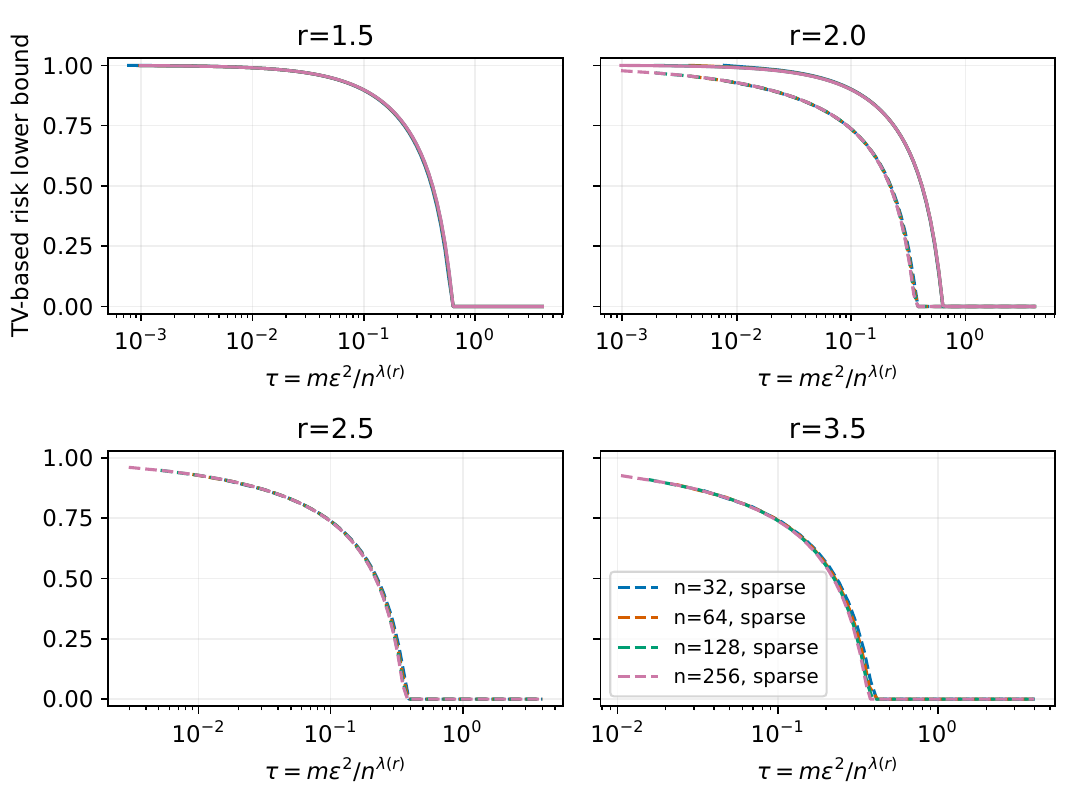}
\caption{The total-variation lower bound $1-\frac12\sqrt{\E_0L^2-1}$ on the minimax risk, evaluated numerically for $\epsilon=0.5$ and $n\in\{32,64,128,256\}$ at the actual rounded sample sizes. Solid curves use the exact second moment of the dense construction, and dashed curves the exact hypergeometric moment of the sparse construction with $k=\lfloor n/2\rfloor$. Both appear at $r=2$. All 2,000 values are computed in log space, and none of them exceeds its analytic upper bound.}
\label{fig:lowerbounds}
\end{figure}

\FloatBarrier
\subsection{Structured models}
Figure~\ref{fig:structured} compares the theorem test with the oracle-calibrated union on the structured models of Appendix~\ref{app:structured}, and Figure~\ref{fig:structured-null} shows their null rejection rates. The theorem thresholds depend on $\epsilon$, so the theorem test's null rejection changes with $\epsilon$ even though the null does not. The two figures show $r=2$, $2.5$, and $3.5$. The $r=4$ cells, where the theorem uses the single $L_4$ statistic, enter only the calibration summary in Table~\ref{tab:practical}. Figure~\ref{fig:reference-null} shows the reduced reproduction of the fourth-norm null experiment of \citet{chatterjee2023}.

\begin{figure}[ht]
\centering\includegraphics[width=\linewidth]{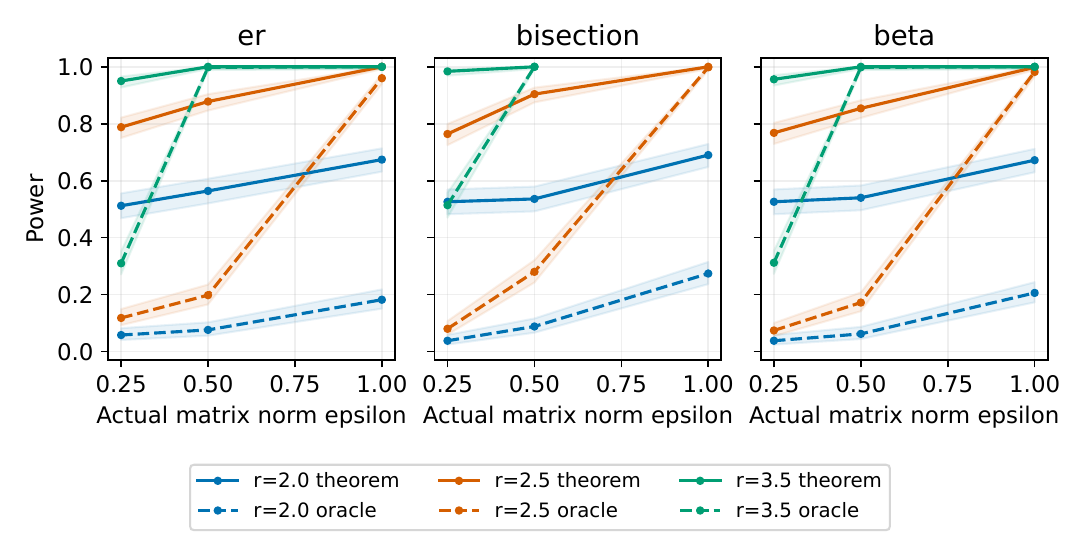}
\caption{Power on the structured models at $n=50$, $m=48$, and $\epsilon\in\{0.25,0.5,1\}$. Solid lines use the theorem test (a union for non-integral $r>2$) and dashed lines the oracle-calibrated union. Colors identify $r$. Each cell has 500 evaluation and 999 calibration replicates, and bands are pointwise 95\% Wilson intervals, conditional on the calibration sample for the oracle.}
\label{fig:structured}
\end{figure}

\begin{figure}[ht]
\centering\includegraphics[width=\linewidth]{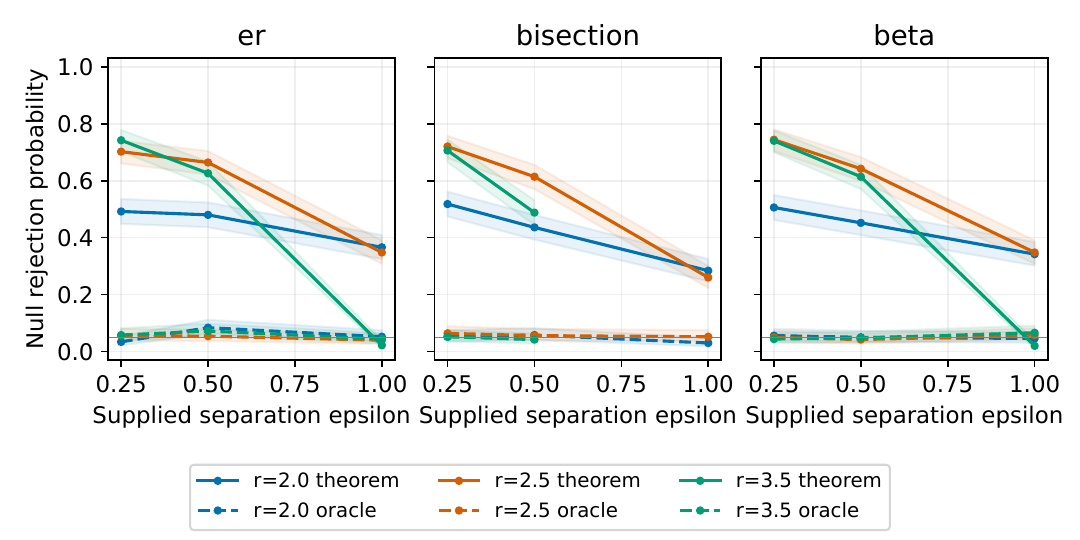}
\caption{Null rejection rates on the structured models at $n=50$, $m=48$. Solid lines show the theorem test and dashed lines the oracle-calibrated union, with colors for $r=2,2.5,3.5$. Each cell uses 500 evaluation and 999 calibration replicates, and bands are pointwise 95\% Wilson intervals. The horizontal line marks $0.05$.}
\label{fig:structured-null}
\end{figure}

\begin{figure}[ht]
\centering\includegraphics[width=\linewidth]{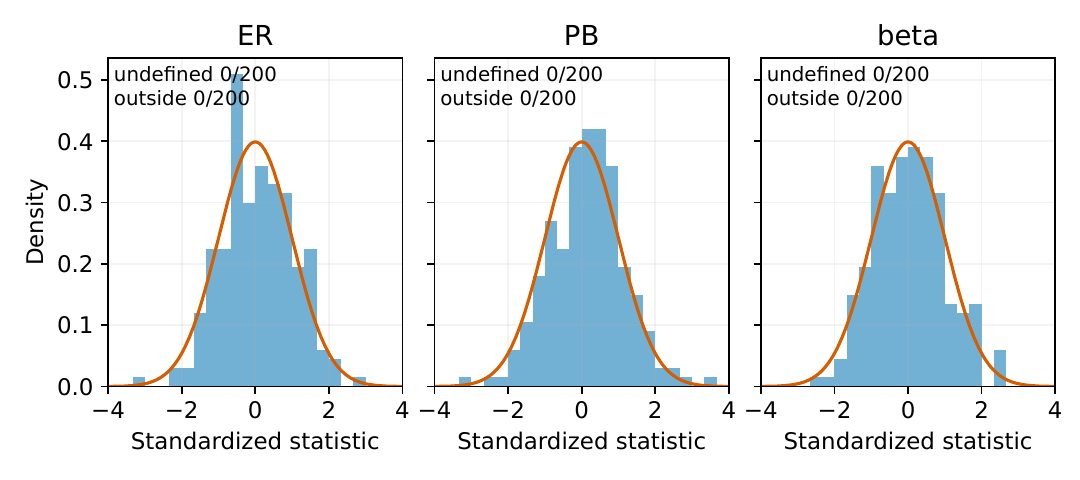}
\caption{Standardized statistic $U_4/\sqrt{\widehat V_4}$ under the null at $n=100$, $m=32$, with 200 replicates per model. Histograms are normalized by the number of finite values, and undefined or out-of-range values are counted separately. The orange curve is the standard normal density. The planted bisection with probabilities $8/n$ and $2/n$ violates the norm-divergence condition of the normal approximation.}
\label{fig:reference-null}
\end{figure}

\FloatBarrier
\subsection{Practical normal calibration}
Table~\ref{tab:practical} summarizes the null rejection rates of the practical normal calibration over all null cells of the main and structured studies. The ranges cover different sample sizes and backgrounds.

\begin{table}[ht]\centering
\caption{Null rejection rates of the practical normal calibration at nominal level $0.05$, over all null cells of the main and structured studies. Rates are computed among the defined decisions, and the last column counts undefined decisions among all replicate evaluations.}\label{tab:practical}
\begin{tabular}{rlrcc}\toprule
$r$ & Method & Cells & Observed range & Undefined/total\\\midrule
1.5 & $L_2$ & 32 & 0.034--0.062 & 0/20000 \\
1.5 & union & 32 & 0.034--0.062 & 0/20000 \\
2 & $L_2$ & 41 & 0.034--0.066 & 0/24500 \\
2 & union & 41 & 0.034--0.066 & 0/24500 \\
2.5 & $L_2$ & 41 & 0.032--0.080 & 0/24500 \\
2.5 & $L_3$ & 41 & 0.034--0.068 & 1/24500 \\
2.5 & union & 41 & 0.034--0.074 & 1/24500 \\
3.5 & $L_2$ & 39 & 0.016--0.078 & 0/23500 \\
3.5 & $L_4$ & 39 & 0.030--0.078 & 144/23500 \\
3.5 & union & 39 & 0.036--0.082 & 144/23500 \\
4 & $L_2$ & 8 & 0.030--0.062 & 0/4000 \\
4 & $L_4$ & 8 & 0.034--0.060 & 0/4000 \\
4 & union & 8 & 0.038--0.056 & 0/4000 \\
\bottomrule\end{tabular}\end{table}

The pooled fraction of undefined decisions hides that they occur only at the smallest sample sizes. At $r=3.5$ and $n=32$, the union is undefined in $78/500=15.6\%$ of null replicates for $m=8$ and in $66/500=13.2\%$ for $m=13$, and these two cells account for all 144 undefined decisions at $r=3.5$. Among the defined decisions, the union rejects in $30/422$ and $30/434$ replicates, respectively. The theorem test does not standardize its statistics and is unaffected.

\end{document}